\documentclass[runningheads]{llncs}

\usepackage{url}
\usepackage{times}
\usepackage{latexsym}
\usepackage{xspace}
\usepackage{amsmath,amssymb,amsfonts}
\usepackage{stmaryrd}
\usepackage{epstopdf}
\usepackage{eurosym}

{\qed \end{trivlist}}

\newenvironment{theorem*}[2]%
{\begin{trivlist} \item[] {\bf #1~\protect{\ref{#2}.}}\it}{\end{trivlist}}

\newcommand{\hide}[1]{}
\newcommand{\tup}[1]{\ensuremath{\langle #1 \rangle}\xspace}

\newcommand{\easycal}[1]{\ensuremath{\mathcal{#1}}\xspace}
\newcommand{\easysf}[1]{\ensuremath{\mathsf{#1}}\xspace}
\newcommand{\easyit}[1]{\ensuremath{\mathit{#1}}\xspace}

\newcommand{\easytt}[1]{\ensuremath{\mathtt{#1}}\xspace}

\usepackage{enumerate}
\usepackage{pgfplots}
\pgfplotsset{compat=1.13}

\newcommand{\SPECIAL}{SPECIAL\xspace}

\usepackage[ruled,linesnumbered]{algorithm2e}

\newcommand{\SRIQ}{\easycal{SRIQ}}

\newcommand{\I}{\easycal{I}}
\newcommand{\J}{\easycal{J}}
\newcommand{\K}{\easycal{K}}

\renewcommand{\L}{\easycal{L}}
\renewcommand{\O}{\easycal{O}}

\newcommand{\T}{\easycal{T}}

\newcommand{\QL}{\easycal{QL}}

\newcommand{\NC}{\easysf{N_C}}
\newcommand{\NR}{\easysf{N_R}}
\newcommand{\NI}{\easysf{N_I}}
\newcommand{\NF}{\easysf{N_F}}

\newcommand{\func}{\easysf{func}}
\newcommand{\range}{\easysf{range}}
\newcommand{\disj}{\easysf{disj}}

\newcommand{\PL}{\easycal{PL}}
\newcommand{\SSA}{\easysf{STS}}
\newcommand{\PLR}{\easysf{PLR}}
\newcommand{\SSO}{\easysf{STS^\O}}
\newcommand{\PLRO}{\easysf{PLR^\O}}

\newcommand{\purp}{\easysf{has\_purpose}}
\newcommand{\data}{\easysf{has\_data}}
\newcommand{\proc}{\easysf{has\_processing}}
\newcommand{\recip}{\easysf{has\_recipient}}
\newcommand{\stor}{\easysf{has\_storage}}
\newcommand{\loc}{\easysf{has\_location}}
\newcommand{\dur}{\easysf{has\_duration}}
\newcommand{\duty}{\easysf{has\_duty}}

\newcommand{\splt}[2]{\ensuremath{\mathit{split}_{#2}(#1)}\xspace}

\newcommand{\sig}[1]{\Sigma(#1)}
\newcommand{\pos}{\easyit{pos}}
\newcommand{\SROIQ}{\easycal{SROIQ}}

\newcommand{\PLSO}{\easycal{PLSO}}

\renewcommand{\ni}{\easyit{\# int}}

\begin{document}
\title{Fast Compliance Checking with General Vocabularies\thanks{This research is funded by the European Union’s Horizon 2020 research and innovation programme under grant agreement N.~731601.}}
%
%
\author{Piero A.~Bonatti\inst{1,2}\orcidID{0000-0003-1436-5660} \and
Luca Ioffredo\inst{2} \and
Iliana M.~Petrova\inst{1,2}\orcidID{0000-0002-1024-2674} \and
Luigi Sauro\inst{1,2}\orcidID{0000-0001-6056-0868}
}
\authorrunning{P.\ Bonatti et al.}
%
\institute{Universit\`a di Napoli Federico II \\
  \email{\{pab,luigi.sauro\}@unina.it}\and
CeRICT\\
\url{http://www.cerict.it/}
}
\maketitle              
\begin{abstract}
We address the problem of complying with the GDPR while processing
and transferring personal data on the web. For this purpose we introduce an extensible
profile of OWL2 for representing data protection policies.  With this language, a company's data usage policy can be
checked for compliance with data subjects' consent and with a formalized fragment of
the GDPR by means of subsumption queries.  The outer structure of the
policies is restricted in order to make compliance checking highly
scalable, as required when processing high-frequency data
streams or large data volumes. However, the vocabularies for specifying policy properties
can be chosen rather freely from expressive Horn fragments of
OWL2.
We exploit IBQ reasoning to
integrate specialized reasoners for the policy language and the
vocabulary's language. Our experiments  show that this approach
significantly improves performance.
\keywords{GDPR  \and Policy language \and IBQ.}
\end{abstract}
\section{Introduction}
\label{sec:intro}

The European General Data Protection Regulation (GDPR) constrains the
use of the personal data of European citizens, no matter where the
\emph{controller} (i.e.\ the entity that collects and processes the
data) is located.  Violations may have severe consequences, such as
significant economic sanctions (4\% of worldwide turnover) and loss of
reputation. Therefore companies are looking for methodologies and
technologies that support compliance with the GDPR. The H2020 project
SPECIAL addresses these needs in several ways, including a
policy-aware framework consisting of:\footnote{For more details on the overall approach of SPECIAL, see \cite{DBLP:conf/esws/KirraneFDMPBWDR18,DBLP:conf/bigdata/BonattiK19}.}
\begin{itemize}
\item A semantic policy language for expressing: (i) \emph{business
  policies}, i.e.\ the data usage policies adopted by the controller;
  (ii) the \emph{consent} to data processing granted by the data subjects;
  (iii) an axiomatization of the ``objective'' part of the GDPR.
\item A compliance checker capable of verifying whether business
  policies are compatible with the available consent and with the
  formalized fragment of the GDPR.
\item Explanation facilities for the data subjects and for policy authors.
\end{itemize}
The semantic policy language consists of a profile of OWL2 called \PL,
and vocabularies for expressing policy properties such as the data
categories involved in the processing, the nature and purpose of the
processing, the recipients of the results, and information about where
and how long data are stored. Compliance checking is currently reduced
to subsumption checking in \PL \cite{DBLP:conf/ijcai/Bonatti18}.

While compliance with respect to the GDPR is a validation phase that
takes place before deploying business policies, compliance checks
w.r.t.\ consent occur at run-time, in such a way as to allow data
subjects to modify or withdraw their consent anytime.  Moreover, many
interesting web-based scenarios involve the processing of large
volumes of personal data  (including various forms of personal data collection, akin to tracking and fingerprinting) that require scalable, possibly real-time
compliance checking w.r.t.\ the consent policies released by all the
subjects which the data refer to.
For these reasons SPECIAL developed a scalable reasoner of \PL, called PLR, supporting
$10^3$-$10^4$ compliance checks per second.

The expressiveness of \PL is appropriate for encoding the structure of
data usage policies, licences, and even EHR; however it is too limited
for the needs of the associated vocabularies of properties,
considering that concept inclusions in \PL are restricted to class
names only.  The vocabularies for encoding data usage and GDPR
concepts are being developed independently by the \emph{Data Privacy
  Vocabularies and Controls Community Group}'' (DPVCG) of the
W3C.\footnote{\url{www.w3.org/community/dpvcg/}} We intend to put as
few constraints as possible on the development of such standardized
vocabularies, because as more application domains are introduced in
the vocabularies and more standards (e.g.\ for classifying
controllers) are imported, it is difficult to predict the
expressiveness needs that may arise in their modeling.

In this paper, we address this need by introducing a flexible method
for integrating \PL-based compliance checking with more general
vocabularies, expressed in Horn-\SRIQ, leveraging the literature on
\emph{import by query} reasoning (IBQ). The IBQ approach makes it
possible to integrate PLR with specialized engines for the logic of
the vocabularies, thereby improving the performance of the available
reasoners, and addressing the aforementioned scalability requirements.

In the next section, we start with technical preliminaries. In
Section~\ref{sec:IBQ} we apply (and slightly extend) the theory of IBQ
to integrate \PL knowledge bases with external ontologies expressed in
Horn-\SRIQ (and fragments thereof). Then, in
Section~\ref{sec:experiments}, we report our experiments.
Software  and data can be downloaded from \url{https://1drv.ms/u/s!Aple1sNCCRUesOEzvzGVRqZnS3uU0Q?e=JHzsyF}.


\section{Preliminaries}
\label{sec:prelim}

We assume the reader is familiar with the basics on Description Logics (DL) \cite{DBLP:conf/dlog/2003handbook}, here we focus only on the aspects needed for this work.
The DL languages of our interest are built from countably infinite sets of concept names (\NC), role names (\NR), individual names (\NI), concrete property names (\NF). 
A signature $\Sigma$ is a subset of $\NC\cup\NR\cup\NI\cup\NF$.
An \emph{interpretation} \I of a signature $\Sigma$ is a structure
$\I=(\Delta^\I, \cdot^\I)$ where $\Delta^\I$ is a nonempty set, and
the \emph{interpretation function} $\cdot^\I$, defined over
$\Sigma$, is such that (i)~$A^\I \subseteq \Delta^\I$ if
$A\in\NC$; (ii)~$R^\I \subseteq \Delta^\I\times\Delta^\I$ if
$R\in\NR$; (iii)~$a^\I\in\Delta^\I$ if $a\in\NI$; (iv)~$f^\I
\subseteq \Delta^\I\times \mathbb{N}$ if $f\in\NF$, where
$\mathbb{N}$ denotes the set of natural numbers.

Compound concepts and roles are built from concept names, role names, and the logical constructors listed in Table~\ref{tab:syntax-semantics}.  
We will use metavariables $A,B$ for concept names,  
$C,D$ for (possibly compound) concepts, $R,S$ for (possibly inverse) roles, $a,b$ for individual names, and $f,g$ for concrete property names.  
The third column  shows how to
extend the valuation $\cdot^\I$ of an interpretation \I to compound
expressions.  
Table~\ref{tab:syntax-semantics} also shows the terminological and assertional axioms we deal with. An interpretation \I \emph{satisfies} an axiom $\alpha$ (in symbols, $\I\models \alpha$) if it satisfies the corresponding semantic condition in Table~\ref{tab:syntax-semantics}.
As usual, $C\equiv D$ is an
abbreviation for the pair of inclusions $C\sqsubseteq D$ and
$D\sqsubseteq C$. Similarly, $\disj(C,D)$, $\range(R,C)$ and $\func(R)$ are abbreviations for $C\sqcap D\sqsubseteq \bot$, $\top\sqsubseteq \forall R.C$ and $\top\sqsubseteq (\le 1 R.\top)$, respectively.

A \emph{knowledge base} \K is a finite set of DL axioms.  Its \emph{terminological part} (or \emph{TBox}) is the set of terminological axioms in \K, while its \emph{ABox} is the set of its assertion axioms.
\vspace*{-20px}
\begin{table}[h]
    \caption{Syntax and semantics of some DL constructs and axioms.}
    \label{tab:syntax-semantics}
\vspace*{-10px}
  \begin{center}
    \leavevmode
    \footnotesize
    \begin{tabular}{p{5.0em}cp{25em}}
      \hline
      \hline
      Name &Syntax&Semantics\\
      \hline
      \hline
      \multicolumn{3}{l}{Compound expressions}
      \\
      \hline
      & &\\[-1em]
      inverse & $R^-$ & $\{ (y,x) \mid (x,y)\in R^\I \}$ \quad ($R\in\NR$)
      \\[-2pt]
      role & &
      \\
      & &\\[-1em]
      top&$\top$& $\top^\I=\Delta^\I$ \\
      & &\\[-1em]
      bottom&$\bot$& $\bot^\I=\emptyset$ \\
      & &\\[-1em]
      intersection&$C\sqcap D$&$(C\sqcap D)^\I=C^\I\cap D^\I$\\
      & &\\[-1em]
      union&$C\sqcup D$&$(C\sqcup D)^\I=C^\I\cup D^\I$\\
      & &\\[-1em]
      existential
      &$\exists R. C$& 
      $\{ d \in \Delta^\I\mid \exists (d,e) \in R^\I : e\in C^\I \}$\\[-2pt]
      restriction & &
      \\
      universal
      &$\forall R. C$& 
      $\{ d \in \Delta^\I\mid \forall (d,e) \in R^\I : e\in C^\I \}$\\[-2pt]
      restriction & &
      \\
      number & $({\bowtie}\,n\ S.C)$ & $\big\{ x\in\Delta^\I \mid
      \#\{y\mid (x,y)\in S^\I \land y\in C^\I\} \bowtie n \big\}$\quad ($\bowtie = \leq, \geq$)
      \\[-2pt]
      restrictions & &
      \\
      self & $\exists S.\mathsf{Self}$ & $\{x\in\Delta^\I \mid (x,x) \in S^\I\}$\\
      
      interval & $\exists f.[\ell,u]$ & $\{x\in\Delta^\I \mid \exists i\in [\ell,u] : (d,i)\in f^\I$\\[-2pt]
      restrictions & &
      \\ [2pt]     
%
      \hline
      \multicolumn{2}{l}{Terminological axioms}
      & \I satisfies the axiom if:
      \\
      \hline
      \\[-1em]
      GCI & $C\sqsubseteq D$ & $C^\I\subseteq D^\I$
      \\[2pt]
      role & $\disj(S_1,S_2)$ & $S_1^\I\cap S_2^\I = \emptyset$
      \\[-2pt]
      disjointness & &
      \\[2pt]
      complex & $R_1\circ\!...\!\circ R_n \sqsubseteq R$ &
      $R_1^\I\circ\ldots\circ R_n^\I \subseteq R^\I$
      \\[-2pt]
      \multicolumn{2}{l}{role inclusions} &
      \\[2pt]
%

      \hline
      \multicolumn{3}{l}{Concept and role assertion axioms  \quad ($a,b\in\NI$)}
      \\
      \hline
      \\[-1em]
      conc.\ assrt. & $C(a)$ & $a^\I \in C^\I$
      \\
      role assrt. & $R(a,b)$ & $(a,b)^\I \in R^\I$
      \\[2pt]
      \hline 

    \end{tabular}
    \end{center}  
\end{table}
\vspace*{-20px}
%

If $X$ is a DL expression, an axiom, or a knowledge base, then $\sig{X}$ denotes the signature consisting of all symbols occurring in $X$.  An interpretation \I of a signature $\Sigma\supseteq\sig{\K}$ is a \emph{model} of \K (in symbols, $\I\models\K$) if \I satisfies all the axioms in \K. We say that \K \emph{entails} an axiom $\alpha$ (in symbols, $\K\models\alpha$) if all the models of \K satisfy $\alpha$.
Given a knowledge base $K$ and general concept inclusion (GCI) 
$C\sqsubseteq D$, the \emph{subsumption problem} consists in deciding whether $\K\models C\sqsubseteq D$.
A \emph{pointed interpretation} is a pair $(\I,d)$ where $d\in\Delta^\I$. We say $(\I,d)$ \emph{satisfies} a concept $C$  iff $d\in C^\I$.  In this case, we write $(\I,d)\models C$.

It is straightforward to see that any knowledge base $\K$ defined on the base of Table~\ref{tab:syntax-semantics} satisfies the   
\emph{disjoint model union property}, that is, if two disjoint interpretations \I and \J satisfy \K, their disjoint union $\I\uplus\J=\tup{\Delta^{\I}\uplus\Delta^{\J},\cdot^{\I\uplus \J}}$ -- where $P^{\I\uplus\J}=P^\I \uplus P^\J$ for all $P\in \NC\cup\NR\cup\NF$ --  satisfies \K, too (\cite{DBLP:conf/dlog/2003handbook}, Ch.~5). This result can be easily extended to the union $\biguplus S$ of an arbitrary set $S$ of disjoint models. 

A knowledge base $\K$ is \emph{semantically modular} with respect to a signature $\Sigma$ if each interpretation $\I=(\Delta^\I, \cdot^\I)$ over $\Sigma$ can be extended to a model $\J=(\Delta^\J, \cdot^\J)$ of $\K$ such that $\Delta^\J=\Delta^\I$ and $X^\J=X^\I$, for all symbols $X\in \Sigma$. Roughly speaking, this means that \K does not constrain the symbols of $\Sigma$ in any way.
A special case of semantic modularity exploited in  \cite{DBLP:conf/ijcai/GrauMK09} is \emph{locality}: 
  A knowledge base $\K$ is \emph{local} with respect to a signature
  $\Sigma$ if each interpretation $\I=(\Delta^\I, \cdot^\I)$ over
  $\Sigma$ can be extended to a model $\J=(\Delta^\J, \cdot^\J)$ of
  $\K$ by setting $X^\I=\emptyset$ for all concept and role names
  $X\in \Sigma(\K)\setminus\Sigma$.

Finally, a Horn-\SRIQ knowledge base 
\cite{DBLP:conf/kr/OrtizRS10,DBLP:conf/ijcai/OrtizRS11} consists of terminological and assertional axioms from  Table~\ref{tab:syntax-semantics} satisfying the following restrictions: 
(i) the set of role axioms should be
\emph{regular} and the roles $S,S_1,S_2$ \emph{simple}, according to
the definitions stated in \cite{DBLP:conf/kr/HorrocksKS06}\footnote{The definitions are omitted because they are not needed in our results.}, and (ii) GCIs have the following normal form:
\begin{equation*}
C_1 \sqcap C_2 \sqsubseteq D\quad\quad
\exists R.C \sqsubseteq D	\quad\quad
C \sqsubseteq \forall R.D	
\end{equation*}
\vspace*{-20px}
\begin{equation*}
C \sqsubseteq \exists R.D\quad\quad
C \sqsubseteq {} (\leq 1\ S.D)\quad\quad
C \sqsubseteq {} (\geq n\ S.D)\, ,
\end{equation*}
where $C,C_1,C_2,D$ either belong to $\NC \cup \{\bot,\top\}$
or are of the form $\exists S.\mathsf{Self}$, and    
$S$ is a \emph{simple} role. 
%
%
%
Like all Horn DLs, Horn-\SRIQ is \emph{convex}, that is, $\K\models C_0\sqsubseteq C_1\sqcup C_2$ holds iff either $\K\models C_0\sqsubseteq C_1$ or  $\K\models C_0\sqsubseteq C_2$.

\section{Semantic Encoding of Data Usage Policies}
\label{sec:enc}

SPECIAL's policy language \PL\ is a fragment of OWL2-DL that has been
specifically designed to describe data controller/subject usage policies and to model
selected parts of the GDPR that can be used to support the validation
of the controller's internal processes.  
%
The aspects of data usage that have legal relevance are clearly
indicated in several articles of the GDPR and in the available
guidelines:
\begin{itemize}
\item reasons for data processing (purpose);
\item which data categories are involved;
\item what kind of processing is applied to the data;
\item which third parties data are distributed to (recipients);
\item countries in which the data will be stored (location);
\item time constraints on data erasure (duration).
\end{itemize}

\SPECIAL adopts a direct encoding of usage policies in description
logics, based on those features. The simplest possible policies have the form:
{
\begin{equation}
  \label{pol1}
  \renewcommand{\arraystretch}{1}
  \begin{array}{l}
    \exists \purp.P \sqcap \exists \data.D \sqcap \exists \proc.O
    \sqcap \exists \recip.R \sqcap {}\\
    ~~ \exists \stor(\exists \loc.L \sqcap \exists \dur.[t_1,t_2]) \,.
  \end{array}
\end{equation}%
}%
All of the above roles are functional.
Duration is represented as an interval of integers
$[t_1,t_2]$, representing a minimum and a maximum storage time (such
bounds may be required by law, by the data subject, or by the controller itself).
The classes $P$, $D$, $O$, etc.\ are defined in suitable
\emph{auxiliary} knowledge bases that specify  the
relationships between different terms.  

A policy of the form (\ref{pol1}) may represent the conditions under which  a data subject gives her consent. For example if $D=\mathsf{DemographicData}$ then the data subject authorizes -- in particular -- the use of her address, age, income, etc.\ as specified by the other properties of the policy.

The usage policies that are applied by the data controller's
business processes are called \emph{business policies} and include a
description of data usage of the form (\ref{pol1}).  Additionally,
each business policy is labelled with its legal basis and describes
the associated obligations that must be fulfilled. For example, if the
data category includes personal data, and processing is allowed by
explicit consent, then the business policy should have the additional
conjuncts:
{
  \begin{equation}
  \label{pol3}
  \renewcommand{\arraystretch}{1}
  \begin{array}{l}
    \exists \mathsf{has\_legal\_basis. Art6\_1\_a\_Consent} \sqcap {} \\
    ~ ~
    \exists\duty.\mathsf{GetConsent} \sqcap \exists\duty.\mathsf{GiveAccess} \sqcap {}\\
    ~ ~
    \exists\duty.\mathsf{RectifyOnRequest} \sqcap 
    \exists\duty.\mathsf{DeleteOnRequest} 
  \end{array}
  \end{equation}%
}%

\noindent
that label the policy with the chosen legal basis from Art. 6 GDPR, and model the obligations related to the data subjects' rights, cf.\ Chapter 3 of the GDPR. More precisely, the terms involving \duty assert that the process modelled by the business policy includes the operations needed to obtain the data subject's consent ($\exists\duty.\easysf{GetConsent}$) and those needed to receive and apply the data subjects' requests to access, rectify, and delete their personal data.
%

Data controllers and subjects may specify different policies for  different categories of data and different purposes. The result is a \emph{full policy}
$P_1 \sqcup \ldots \sqcup P_n$ where each  $P_i$ is a
\emph{simple usage policies} like (\ref{pol1}) or (\ref{pol1}) $\sqcap$ (\ref{pol3}) (one for each usage type).

Simple usage policies are formalized by \emph{simple \PL concepts}, that are defined by the following grammar,
  where $A\in\NC$, $R\in\NR$, and $f\in\NF:$
  $$C,D::=A\mid \bot\mid \exists f.[l,u] \mid \exists R.C \mid C \sqcap D \,.$$
  A \emph{(full) \PL concept} is a union $D_1\sqcup\ldots\sqcup D_n$
  of simple \PL concepts ($n\geq 1$).

In order to check whether a business process complies with the consent
given by a data subject $S$, it suffices to check whether the
corresponding business policy \easyit{BP} is subsumed by the consent
policy of $S$, denoted by \easyit{CP_S} (in symbols,
\easyit{BP\sqsubseteq CP_S}). This subsumption is checked against a
knowledge base that encodes type restrictions related to policy
properties.  
Such a knowledge base is the union of a \emph{main} knowledge base $\K$ that specifies the semantics of general terms, such as $\purp$ or $\data$,  plus the \emph{auxiliary} knowledge base $\O$ that models the different types of data, purposes, recipients, etc. according to a specific application domain. 
An example of the actual axioms occurring in $\K$ is:
\[
\renewcommand{\arraystretch}{1}
\begin{array}{ll}
  \easysf{\func(\purp)} &   \easysf{\range(\purp,AnyPurpose)} \\
  \easysf{\func(\data)} &    \easysf{\range(\data,AnyData)} \\
  \easysf{\disj(AnyData, AnyPurpose)}\, .
\end{array}
\]

\noindent
Formally, we assume that a \emph{main} \PL\ knowledge base $\K$ is a set of axioms of the following kinds:
  \begin{itemize}
  \item $\func(R)$ where $R$ is a role name or a concrete property;
  \item $\range(S,A)$ where $S$ is a role and $A$ a concept name;
  \item $A\sqsubseteq B$ where $A,B$ are concept names;
  \item $\disj(A,B)$ where $A,B$ are concept names.
  \end{itemize}
In \cite{DBLP:conf/ijcai/Bonatti18} \O is expressed in the same way. In the following we are showing how to support more general auxiliary knowledge bases expressed in Horn-\SRIQ.

\section{Supporting General Vocabularies with IBQ Reasoning}
\label{sec:IBQ}

Our strategy consists in treating the auxiliary ontologies as
\emph{oracles}.  Roughly speaking, whenever the reasoner for \PL needs
to check a subsumption between two terms defined in the auxiliary
ontologies, the subsumption query is submitted to the oracle.  Of
course this method, called \emph{import by query} (IBQ), is not always
complete
\cite{DBLP:conf/ijcai/GrauMK09,DBLP:journals/jair/GrauM12}. In the
following, we provide sufficient conditions for completeness.

In SPECIAL's policy modeling scenario, the \emph{main} ontology \K
defines policy attributes, such as data categories, purpose etc.\ --
by specifying their ranges and functionality properties -- while the
\emph{auxiliary} ontology \O defines the privacy-related vocabularies
that provide the range for those attributes.
The reasoning task of interest in such scenarios is deciding, for a
given subsumption query $q=(C\sqsubseteq D)$, whether $\K \cup \O
\models q$.  Both $C$ and $D$ are \PL concepts (policies) that may
contain occurrences of concept names defined in \O.

SPECIAL's application scenarios make it possible to adopt a
simplifying assumption that makes oracle reasoning technically simpler
\cite{DBLP:conf/ijcai/GrauMK09,DBLP:journals/jair/GrauM12}, namely, we
assume that neither \K nor the query $q$ share any roles with \O.
This naturally happens in SPECIAL precisely because the roles used in
the main KB identify the sections that constitute a policy (e.g.\ data
categories, purpose, processing, storage, recipients), while the roles
defined in \O model the \emph{contents} of those sections,
e.g.\ anonymization parameters, relationships between recipients,
relationships between storage locations, and the like.

The IBQ framework was introduced to reason with a partly hidden
ontology \O.  For our purposes, IBQ is interesting because instead of
reasoning on $\K\cup \O$ as a whole, each partition can be processed more efficiently
with a different, specialized reasoner.  The reasoner for
$\K$ may query  $\O$ as an oracle, using a query
language \QL consisting of all the subsumptions
\begin{equation}
\label{oq}
A_1 \sqcap \ldots \sqcap A_m \sqsubseteq A_{m+1} \sqcup \ldots
\sqcup A_n 
\end{equation}
such that $A_1, \ldots, A_n$ are concept names.
We will denote with $\pos(\O)$ all the queries to \O that have a positive answer, that is:
\begin{equation*}
\pos(\O) = \{q \in \QL  \mid \O\models q\}\,.
\end{equation*}

\noindent
The problem instances of our interest are formally defined as follows:

\begin{definition}[\PL subsumption instances with oracles, \PLSO]
  \label{def:ibq}
  A \emph{\PL subsumption instance with oracle} is a triple
  \tup{\K,\O,q} where $\K$ is a \PL knowledge base (the \emph{main
    knowledge base}), $\O$ is a Horn-\SRIQ knowledge base (the
  \emph{oracle}), and $q$ is a \PL subsumption query such that
  $(\sig{\K}\cup\sig{q})\cap\sig{\O}\subseteq \NC$.  The set of all
  \PL subsumption instances with oracle will be denoted by \PLSO.
\end{definition}

\noindent
The restriction on the signatures is aimed at
keeping the roles of \O separated from those of \K and $q$, as
previously discussed. 
Oracles are restricted to Horn-\SRIQ because (to the best of our
knowledge) it is the most expressive convex and nominal-free logic
studied so far in the literature; it can be proved that convexity is
essential for the tractability of \PL with oracles, and that nominals
affect the completeness of IBQ reasoning (we do not include these
results here due to space limitations).

The next lemma rephrases the original completeness result for  IBQ
\cite[Lemma~1]{DBLP:conf/ijcai/GrauMK09} in our notation.  Our statement
relaxes the requirements on \O by assuming only that it enjoys the
disjoint model union property (originally it had to be in
\easycal{SRIQ}).  The proof, however, remains essentially the same, and is omitted here.

\begin{lemma}
  \label{lem:CT-completeness-altrui}
  Let \K and \O be knowledge bases and $\alpha$ a GCI, such that
  \begin{enumerate}
  \item \K and $\alpha$ are in $\SROIQ(\easycal{D})$ without the
    universal role,\footnote{The definition of this DL is omitted due
      to space limits; the only important thing is that \PL is a
      fragment of $\SROIQ(\easycal{D})$ without universal role.}
    where \easycal{D} is the concrete domain of integer intervals;
  \item The terminological part of \O enjoys the disjoint model union property;
  \item The terminological part \T of \K is local w.r.t. $\sig{\O}$;
  \item $(\sig{\K} \cup \sig{\alpha}) \cap \sig{\O}\subseteq \NC$.
  \end{enumerate}
  Then $\K\cup\O\models \alpha$ iff $\K\cup\pos(\O)\models \alpha$\,.
\end{lemma}

\noindent
Using the above lemma, we prove a variant of IBQ completeness for
\PLSO. The locality requirement of
Lemma~\ref{lem:CT-completeness-altrui} is removed by shifting axioms
from \K to \O.

\begin{theorem}
  \label{thm:compl-w-shifting}
  For all problem instances $\pi = \tup{\K,\O,q}\in\PLSO$, let
  \begin{eqnarray*}
    \K^- &=& \{ \alpha \in\K \mid  \alpha = \range(R,A) \mbox{ or } \alpha = \func(R) \,\}\\
    \O^+_\K &=& \O\cup(\K \setminus \K^-) \,.
  \end{eqnarray*}
  Then
  $
  \K\cup\O\models q \mbox{\: iff \:} \K^- \cup \pos(\O^+_\K) \models q \,.
  $
\end{theorem}

\begin{proof}
  Since $\K\cup\O = \K^- \cup \O_\K^+$, it suffices to show that
  \begin{equation*}
    \label{shifting:1}
    \K^-\cup\O_\K^+\models q \mbox{\: iff \:} \K^- \cup
    \pos(\O^+_\K) \models q\,.
  \end{equation*}
  This equivalence can be proved with
  Lemma~\ref{lem:CT-completeness-altrui}; it suffices to show that
  $\K^-$, $\O_\K^+$ and $q$ satisfy the hypotheses of the lemma.
  First note that $\K^-$ is a \PL knowledge base and $\O_\K^+$ is a
  Horn-\SRIQ knowledge base (because, by definition of \PLSO, \K is in
  \PL and \O in Horn-\SRIQ, and the axioms shifted from \L to
  $\O_\K^+$ can be expressed in Horn-\SRIQ, too).  Both \PL and
  Horn-\SRIQ are fragments of $\cal SROIQ(D)$ without $U$, therefore
  hypothesis~1 is satisfied by $\K^-$ and $\O_\K^+$.  Moreover,
  both \PL and Horn-\SRIQ enjoy
  the disjoint model union property, therefore hypothesis~2 is
  satisfied. Next recall that
  $(\sig{\K}\cup\sig{q})\cap\sig{\O}\subseteq \NC$ holds, by definition of
  \PLSO. Since the axioms $\alpha\in\K\setminus\K^-$ (transferred
  from \K to $\O_\K^+$) contain no roles (they are of the form
  $A\sqsubseteq B$ or $\disj(A,B)$), it follows that
  \begin{equation*}
    (\sig{\K^-}\cup\sig{q})\cap\sig{\O_\K^+}\subseteq \NC \,,
  \end{equation*}
  that is, hypothesis~4 holds.  A second consequence of this inclusion
  is that $\K^-$ contains only axioms of the form $\range(R,A)$ and
  $\func(A)$ such that $R\not\in\sig{\O_\K^+}$.  They are trivially
  satisfied by any interpretation \I such that
  $R^\I=\emptyset$. Therefore $\K^-$ is local
  w.r.t.\ $\sig{\O_\K^+}$ and hypothesis~3 is satisfied.
  \qed
\end{proof}

\noindent
\emph{In the following, let $\K^-$ and $\O_\K^+$ be defined as in
  Theorem~\ref{thm:compl-w-shifting}.}

The reasoner for solving \PLSO consists in two normalization phases followed
by a structural subsumption check.  The first normalization phase
splits the intervals in the left-hand side of the given query
$q$ to make it \emph{interval safe}:

\begin{definition}[Interval safety]
   An inclusion $C\sqsubseteq D$ is \emph{interval safe} iff, for all
   constraints $\exists f.[\ell,u]$ occurring in $C$ and all $\exists
   f'.[\ell',u']$ occurring in $D$, either
   $[\ell,u]\subseteq[\ell',u']$, or $[\ell,u]\cap[\ell',u'] =
   \emptyset$.
\end{definition}

\noindent
Assuming that the number of intervals in each simple \PL concept in
$C$ is bounded by a constant $c$, this phase can be computed in
polynomial time \cite{DBLP:conf/ijcai/Bonatti18}.  The result is denoted by \splt{C}{D}.

The second normalization exhaustively applies the rewrite rules in Table~\ref{norm-rules-O} to the left-hand side of
$C$.  It is easy to see that these rules preserve concept equivalence.  We say that a \PL
concept $C$ is \emph{normalized w.r.t.\ \K and \O} if none of the
rules in Table~\ref{norm-rules-O} is applicable.  
Differently from \cite{DBLP:conf/ijcai/Bonatti18}, rule~7 queries the
oracle to detect inconsistencies.
\begin{table}[h]
	\caption{Normalization rules for \SSO. Conjunctions are
          treated as sets (i.e.\ the ordering of conjuncts is
          irrelevant, and duplicates are removed). }
	\label{norm-rules-O}
	\centering
	\small
	\framebox{
                \renewcommand{\arraystretch}{1.3}
		\begin{minipage}{.96\textwidth}			
			\begin{tabular}{rlp{15em}}
				1) & $\bot \sqcap D \leadsto \bot$
				\\
				2) & $\exists R.\bot \leadsto \bot$
				\\
				3) & $\exists f.[l,u] \leadsto \bot$ & \hspace*{68pt} if $l>u$
				\\
				4) & $(\exists R.D)\sqcap (\exists R.D') \sqcap D'' \leadsto \exists R.( D\sqcap D')\sqcap D''$ 
				& \hspace*{68pt} if $\func(R) \in \K^-$
				\\
				5) & \multicolumn{2}{l}{$\exists f.[l_1,u_1] \sqcap \exists f.[l_2,u_2] \sqcap D \leadsto \exists f.[\max(l_1,l_2),\min(u_1,u_2)]\sqcap D $  \quad  if $\func(f) \in \K^-$}
				\\
				6) & \multicolumn{2}{l}{$\exists R.D\sqcap D' \leadsto \exists R.( D\sqcap A)\sqcap D'$ \hspace*{30pt} if $\range(R,A)  \in \K^-$ and neither $A$ nor $\bot$}  \vspace*{-5pt} \\
                                & \multicolumn{2}{l}{\hspace*{160pt}are conjuncts of $D$}
				\\
				7) & $A_1 \sqcap\ldots\sqcap A_n \sqcap D \leadsto \bot$
				& if $\O_\K^+\models A_1 \sqcap\ldots\sqcap A_n \sqsubseteq \bot$
			\end{tabular}
		\end{minipage}
	}
\end{table}

The third phase of the reasoning is described in
Algorithm~\ref{alg:sso}, that differs from its counterpart for \PL
\cite{DBLP:conf/ijcai/Bonatti18} in line~3, where subsumptions are
checked by invoking the oracle.
Algorithm~\ref{alg:sso} accepts \emph{elementary}
(i.e.\ normalized) concepts:

\begin{definition}
  A \PL subsumption $C\sqsubseteq D$ is \emph{elementary w.r.t.\ \K
    and \O} if both $C$ and $D$ are simple, $C\sqsubseteq D$ is
  interval safe, and $C$ is normalized w.r.t.\ \K and \O.
\end{definition}

\begin{algorithm}[h]
  \caption{$\SSO(C\sqsubseteq D)$}
  \label{alg:sso}
  \small
  \KwIn{An ontology \O and a \PL subsumption $C\sqsubseteq D$ \hide{that is  elementary w.r.t.\ \K and \O}}
  \KwOut{ \easytt{true}\, if $\O \models C\sqsubseteq D$, \quad\easytt{false}\,
    otherwise, \quad under suitable restrictions }
  \vspace*{\medskipamount}
  \Begin{
      \lIf{$C=\bot$}{ \Return{\easytt{true}}  }
      
      \lIf{$D=A$ and $(A_1\sqcap\ldots\sqcap A_n \sqsubseteq A) \in \pos(\O)$, 
        where $A_1,\ldots,A_n$ are the top-level concept names in $C$}{
        \Return{\easytt{true}} }
      
      \lIf{$D=\exists f.[l,u]$  and $C=\exists f.[l',u']\sqcap C'$ and $l\leq l'$ and $u'\leq u$}
          { \Return{\easytt{true}}  }
          
      \lIf{$D=\exists R.D'$, $C=(\exists R.C')\sqcap C''$ and $\SSO(C'\sqsubseteq D')$}
          { \Return{\easytt{true}}  }
              
      \lIf{$D=D'\sqcap D''$, $\SSO(C\sqsubseteq D')$, and $\SSO(C\sqsubseteq D'')$}
                  { \Return{\easytt{true}} }
                  
      \lElse{ \Return{\easytt{false}} }
  }
\end{algorithm}

\noindent
Finally, Algorithm~\ref{alg:main-O} (\PLRO) specifies the complete
reasoning process for general \PL subsumptions with oracles.
 \PLRO can be proved to be sound and complete by analogy with
 \cite{DBLP:conf/ijcai/Bonatti18}.  First, it is possible to define a
 canonical model $(\I,d)$ with the following property:\footnote{The
   definition is similar to \cite[Def.~2]{DBLP:conf/ijcai/Bonatti18};
   each occurrence of $A_i\sqsubseteq^*A$ is replaced by
   $\big(\bigsqcap_{i=1}^nA_i \sqsubseteq A\big) \in \pos(\O_\K^+)$.}

\begin{algorithm}[h]
  \caption{$\PLRO(\K,C\sqsubseteq D)$}
  \label{alg:main-O}
  \small
  \KwIn{\K and $C\sqsubseteq D$ such that $\pi=\tup{\K,\O,C\sqsubseteq D}\in\PLSO$}
  \KwOut{ \easytt{true}\, if $\K \cup \O \models C\sqsubseteq D$, \quad\easytt{false}\, otherwise}
  \vspace*{\medskipamount}
  \Begin{
      {\bf construct}  $\K^-$ and $\O_\K^+$ as defined in Theorem~\ref{thm:compl-w-shifting} \;

      {\bf let} $C'$ be the normalization of $C$ w.r.t.\ \K and \O (with the
      rules in Table~\ref{norm-rules-O}) \;
      
      {\bf let} $C'' = \splt{C'}{D}$  \; 

      \tcp{assume that $C''=C_1\sqcup\ldots\sqcup C_m$ and $D=D_1\sqcup\ldots\sqcup D_n$}

      \tcp{check whether each $C_i$ is subsumed by some $D_j$}

      \For{$i=1,\dots,m$}{
        \For{$j=1,\dots,n$}{
          \lIf  {$\SSA^{\O_\K^+}(C_i\sqsubseteq D_j)=\easytt{true}$}
                {skip to next $i$ in outer loop}
        }
        \Return{\easytt{false}}
      }
      \Return{\easytt{true}}
  }
\end{algorithm}

\begin{lemma}
  \label{lem:mod-of-norm-C}
  If $C$ is a simple \PL concept normalized w.r.t.\ \K and \O, and
  $C\neq\bot$, then each canonical model $(\I,d)$ of $C$ enjoys the
  following properties:
  \begin{enumerate}
    \renewcommand{\theenumi}{\alph{enumi}}
  \item $\I \models \K^-\cup \pos(\O_\K^+)$;
  \item $(\I,d) \models C$.
  \end{enumerate}
\end{lemma}

\noindent
Each canonical model of $C$
characterizes \emph{all} the valid elementary subsumptions whose
left-hand side is $C$:

\begin{lemma}
  \label{lem:counterex}
  If $C\sqsubseteq D$ is elementary w.r.t.\ \K and \O, $C\neq\bot$, and
  $(\I,d)$ is a canonical model of $C$, then $\K^-\cup \pos(\O_\K^+) \models
  C\sqsubseteq D \mbox{\: iff \:} (\I,d)\models D \,.$
\end{lemma}

\noindent
Moreover, by means of canonical models, one can prove that interval safety makes
the non-convex logic \PL behave like a convex logic.

\begin{lemma}
  \label{lem:convexity-PL}
  For all interval-safe \PL subsumption queries $\sigma =
  \big(C_1\sqcup\ldots\sqcup C_m \sqsubseteq D_1 \sqcup\ldots\sqcup
  D_n \big)$ such that each $C_i$ is normalized w.r.t.\ \K and \O, the
  entailment $\K^- \cup \pos(\O_\K^+) \models \sigma$ holds iff for
  all $i\in[1,m]$ there exists $j\in[1,n]$ such that $\K^- \cup \pos(\O_\K^+) \models
  C_i\sqsubseteq D_j$.
\end{lemma}

\noindent
Now that the semantic properties are laid out, we focus on the
algorithms.  Roughly speaking, the next lemma says that
$\SSA^{\O_\K^+}$ (Alg.~\ref{alg:sso}) decides whether the canonical model $(\I,d)$ of $C$
satisfies $D$. The lemma can be proved by structural induction on $D$.

\begin{lemma}
  \label{lem:SSA-evaluates-D}
  If $C\sqsubseteq D$ is elementary w.r.t.\ \K and \O, $C\neq\bot$, and
  $(\I,d)$ is the canonical model of $C$, then $\SSA^{\O_\K^+}(C
  \sqsubseteq D)=\mathtt{true}\mbox{\: iff \:} (\I,d)\models D \,.$
\end{lemma}

\noindent
Finally we can prove that \PLRO (Alg.~\ref{alg:main-O}) is correct and complete.

\begin{theorem}
  \label{thm:PLRO-correct}
  Let \tup{\K,\O,C \sqsubseteq D} be any instance of \PLSO.  Then
  $$\PLRO(\K,C \sqsubseteq D)=\mathtt{true} \mbox{\: iff \:} \K \cup \O \models
  C\sqsubseteq D \,.$$
\end{theorem}

\begin{proof}
  $D$ is of the form $D_1\sqcup\ldots\sqcup D_n$.
  Let $C_1 \sqcup \ldots \sqcup C_m$ be the concept $C''$ computed by lines 2 and 3 of \PLRO.
  We start by proving the following claim, for all $i=1,\ldots,m$ and  $j=1,\ldots,n$:
  \begin{equation}
    \label{PLRO:1}
    \SSA^{\O_\K^+}(C_i\sqsubseteq D_j) = \easytt{true} \mbox{\:
      iff \:} \K^-\cup\pos(\O_\K^+) \models C_i\sqsubseteq D_j \,.
  \end{equation}
  There are two possibilities. If $C_i=\bot$, then clearly $\K^-\cup\pos(\O_\K^+) \models
  C_i\sqsubseteq D_j$ and $\SSA^{\O_\K^+}(C_i \sqsubseteq
  D_j)=\mathtt{true}$ (see line 2 of Algorithm~\ref{alg:sso}), so
  (\ref{PLRO:1}) holds in this case.  If $C\neq\bot$, then note that
  $C_i\sqsubseteq D_j$ is elementary w.r.t.\ \K and \O by construction
  of $C''$ (which is obtained by splitting the intervals of the
  normalization of $C$ w.r.t.\ \K and \O).  Then (\ref{PLRO:1}) follows immediately from
  lemmas~\ref{lem:counterex} and \ref{lem:SSA-evaluates-D}.

  By (\ref{PLRO:1}) and convexity (Lemma~\ref{lem:convexity-PL}), we
  have that lines 5--11 of Algorithm~\ref{alg:main-O} return
  \easytt{true} iff $\K^- \cup \pos(\O_\K^+) \models C''\sqsubseteq
  D$.  Moreover, $C''$ can be equivalently replaced by $C$ in this
  entailment, since normalization preserves equivalence. The resulting entailment is
  equivalent to $\K \cup \O \models C\sqsubseteq D$ by
  Theorem~\ref{thm:compl-w-shifting}.  It follows that
  Algorithm~\ref{alg:main-O} returns \easytt{true} iff $\K \cup \O
  \models C\sqsubseteq D$. \qed
\end{proof}

Using Algorithm~\ref{alg:main-O}, it can be proved that if the number of
intervals occurring in each simple \PL concept is bounded by a
constant $c$ (as in SPECIAL's policies, where $c=1$), then \PL
subsumption checking with oracles is in
$\mathbf{P}^{\pos(\O^+)}$.\footnote{Here we use the notation for complexity classes with oracles adopted in \cite{DBLP:books/daglib/0018514}.} Consequently, if oracles have a tractable subsumption problem, then \PLSO with bounded occurrences of intervals is tractable, too.


\section{Experimental Evaluation}
\label{sec:experiments}

In this section we describe a Java implementation of \PLRO and compare
its performance with that of other popular engines.
\PLRO is implemented in Java and distributed as a .jar file. The
reasoner's class is named \emph{PLReasonerIBQ}, and provides a partial
implementation of the OWL API interfaces, version 5.x.
The package includes a complete implementation of \PLRO, including the
structural subsumption algorithm \SSO and the two normalization
phases, based on the interval splitting method for interval safety
illustrated in \cite{DBLP:conf/ijcai/Bonatti18} and on the rewrite
rules in Table~\ref{norm-rules-O}.
Several optimizations have been implemented. In the following we will assess  
two versions of \PLRO:
    \vspace*{-15pt}
\subsubsection*{IBQ PLR} 
The basic implementation of \PLRO.
    \vspace*{-15pt}
\subsubsection*{IBQ PLR c}
The calls to the oracle (triggered by normalization rule 6
and line~3 of \SSO) are one of the most expensive parts of the
reasoner. In order to reduce their cost, two caches are
introduced, for remembering the results of the oracle queries executed by the normalization phase and  \SSO, respectively.  This optimization is
expected to be effective due to the nature of the interval
normalization phase, that replicates concepts when intervals are split
to achieve interval safety, thereby inducing a large number of identical oracle queries.

\hide{
\subsubsection*{IBQ PLR c+}
The normalization steps (lines~2 and 3 of \PLR) constitute another 
very expensive part of the reasoning algorithm. 
In order to reduce their cost, two
caches are introduced. The first cache stores the business policies that have
already been normalized w.r.t.\ \K (line~2 of \PLR).  In this way, the
seven rewrite rules are applied to each business policy only once;
when the policy is used again, line~2 simply retrieves the normalized
concept from the cache.
This optimization is expected to be effective in SPECIAL's application
scenarios because only business policies need to be normalized, and
their number is limited.  So the probability of
re-using an already normalized policy is high, and the cache is
not going to grow indefinitly; on the contrary its size is expected to be
moderate.
}

\vspace*{5pt} SPECIAL's engine is tested on sets of experiments where
both the main ontologies \K and the \PL subsumptions that encode
compliance checks are completely synthetic.  On the other hand, the
oracles -- namely MHC, OntolUrgences and SNOMED CT\footnote{Note that
  in order to make sure that the selected ontologies do not violate
  the expressivity constraints of Horn-\SRIQ a couple of axioms that
  make use of more expressive DL constructs have been dropped.} --
have been selected from the real ontologies in the BioPortal
repository.  The choice of these ontologies (unrelated to the GDPR
domain) is justified by two considerations:
\begin{itemize}
\item We need large ontologies in order to assess the scalablity of
  the engine in the increasingly complex scenarios we expect in the
  future. The above oracles have been selected due to their size
  (cf. Table~\ref{vocab}) in order to set up a stress test for
  verifying the scalability of \PLRO as vocabularies grow.
\item \PL is well-suited to the representation of electronic health
  records (EHR), using the standard HL7 to represent the structure of
  EHR, and biomedical ontologies (such as SNOMED) to encode the contents
  of each section (cf.\ \cite{DBLP:conf/dlog/BonattiPS15}). So our experiments evaluate the practical behavior of \PL also in the e-health context.
\end{itemize}

\begin{table}[h]
  \caption{Size of the real world general vocabularies}
  \label{vocab}
  \vspace*{-10pt}
  \begin{center}
    \small
    \begin{tabular}{lrrrr}
      \hline
      & \multicolumn{1}{c}{\normalsize \hspace{15pt} MHC } & \multicolumn{1}{c}{\normalsize \hspace{10pt} OntolUrgences} & \multicolumn{1}{c}{\normalsize \hspace{10pt} SNOMED CT}
      \\
      \hline
      \hline
      Oracle size& & &
      
        \\
        \hline
        \hline
        classes & 7,929 & 10,031 & 350,711
        \\
        roles & 8 & 60 & 120
        \\
        concrete properties & 3 & 1 & 0
        \\
        \func & 1 & 4 & 0
        \\
        \range & 5 & 37 & 0
        \\[2pt]
        \hline
        \\[-6pt]
        \disj & 3 & 17 & 0
        \\
        inclusions & 13623 & 12023 & 239871
        \\
        equivalent classes & 16 & 112 & 111262
         \\       
        avg.\ classification & 8.7 & 11.06 & 11.7 \\[-2pt]
         hierarchy height & & &
         \\
        max\ classification & 12 & 16 & 28 \\[-2pt]
         hierarchy height & & &
      \\
      \hline
    \end{tabular}
  \end{center}
\end{table}

The main knowledge bases \K define policy roles (recall that no role
from the external vocabularies may occur in the business and consent
policies, and that class inclusion axioms are all pushed into the
oracles).  Three ontologies have been generated: one for each of the
three sets of parameters K1--K3 reported in
Table~\ref{synthetic-test-cases}. Approximately half of the
roles and concrete properties are functional, and half of the roles
have a range axiom.  The same table reports the parameters P1, P2 and
P3 used to generate the \PL concepts occurring in the queries. Concept size and nesting dominate the corresponding dimensions of the real-world policies
occurring in SPECIAL's pilots. The three sizes P1, P2 and P3 have
different interval length; this parameter influences the probability
of interval splitting during the normalization phase, that has a major effect on complexity.
In particular, interval splitting may exponentially inflate the given
business policy, and this is unavoidable (unless P=NP) because unrestricted \PL
subsumption checking is coNP-complete
\cite{DBLP:conf/dlog/BonattiPS15}.
If the maximum number of intervals per simple policies (hereafter \ni)
is bounded, then the computation of \splt{P_B}{P_C} takes polynomial
time, but the degree of the polyomial grows with \ni.
Although in SPECIAL's data usage policies $\ni \leq 1$ (since storage
duration is the only interval-valued property), in our experiments we
also analyze the costs of higher values of \ni, in view of possible
future extensions.
Note that the value of \ni is measured \emph{after} normalization,
since the rewrite rules may decrease the number of intervals.

\begin{table}[h]
  \caption{Size of fully synthetic test cases}
  \label{synthetic-test-cases}
    \small
    \begin{minipage}[b][140pt][t]{170pt}
      \begin{tabular}{lrrr}
        \hline
        Main ontology size & \normalsize \hspace{15pt} K1 & \normalsize \hspace{10pt} K2  & \normalsize  \hspace{10pt} K3
        \\
        \hline
        \hline
        classes & 0 & 0 & 0
        \\
        roles & 10 & 30 & 50
        \\
        concrete properties & 5 & 10 & 15
        \\
       avg. \func & 5 & 15 & 25
        \\
       max \func & 9 & 27 & 45
        \\
        \range & 5 & 15 & 25
        \\[2pt]
        \hline
        \\[-6pt]
        \disj & 0 & 0 & 0
        \\
        inclusions & 0 & 0 & 0
        \\
        equivalent classes & 0 & 0 & 0
        \\[2pt]
        \hline
      \end{tabular}
    \end{minipage}
    \quad
    \begin{minipage}[b][140pt][t]{145pt}
      \begin{tabular}{lrrr}
        \hline
        Full concept size & \hspace{10pt} P1 & \hspace{10pt} P2 & \hspace{10pt} P3
        \\
        \hline
        \hline
        \#simple pol. per full pol. & 10 & 10 & 10
        \\
        max \#top-level intersec. & 10 & 10 & 10\\[-2pt]
        per simple subconcept
        \\
        depth (nesting) & 4 & 4 & 4
        \\[2pt]
        \hline
        \hline
        Simple policy size
        \\
        \hline
        \\[-6pt]
        \#atomic classes  & 30 & 30 & 30
        \\
        \#exist. restr. per level & 3 & 3 & 3
        \\
        max.\ \#intervals & 8 & 8 & 8 
        \\
        max interval length&  50 & 80 & 150 
        \\[2pt]
        \hline
      \end{tabular}
    \end{minipage}
\end{table}

In each compliance check $P_B \sqsubseteq P_C$, $P_B$ is a union of
randomly generated simple business policies. 
Policy attributes are specified
by randomly  picking classes from the oracles. We make sure that every simple policy 
generated is internally consistent by discarding inconsistent policies. 
The consent policy $P_C$ is the union of a set of simple policies
$P_C^i$ ($i=1,\ldots,n$) generated by modifying the simple business policies in $P_B$, mimicking a selection of
privacy options from a list provided by the controller. In particular,
a random deletion of conjuncts within a simple policy mimicks the opt-out choices of
data subjects with respect to the various data processing activities modelled
by the simple policies.  
Similarly, the random replacement of terms with a different term 
simulates the opt-in/opt-out choices of the data subject w.r.t.\ each
component of the selected simple policies. More precisely, if the
modified term occurring in $P_C^i$ is a superclass (resp.\ a subclass)
of the corresponding term in the original business policy, then the
data subject opted for a broader (resp.\ more restrictive) permission
relative to the involved policy property (e.g.~data categories,
purpose, and so on). Finally, we also consider the random addition 
of new simple policies (disjunct). 

The number of queries for each size P$_i$ generated for each vocabulary is 3600 (50 different queries for each combination of generation parameters).

The experiments have been performed on a server with an 8-cores processor
Intel Xeon Silver 4110, 11M cache, 198\,GB RAM, running Ubuntu 18.04
and JVM 1.8.0\_181, configured with 32GB heap memory. 
We have \emph{not} exploited parallelism in the engine's implementation.

Concerning the engine used to query the oracle, we used the
specialized engine ELK for SNOMED (that is in OWL2-EL) but we had to
use the general engine Hermit on the other two oracles, since they are
too expressive for ELK.\footnote{In future work we are planning to try also
  specialized engines for Horn DLs, such as GraphDB and RDFox.}

Performance is not affected by the choice of K1, K2, or K3, therefore
we aggregate their results. We aggregate also the results for P1, P2,
and P3, because they influence \ni indirectly; we rather focus on the
actual value of \ni after normalization.

The main experimental results are reported in
Table~\ref{tab-comp}. The optimized version of \PLRO is systematically
faster than Hermit for $\ni<5$.  Speedups range approximately from 3
to 6 times. On a very large oracle like SNOMED, \PLRO is almost two
orders of magnitude faster than Hermit, also for higher values of \ni
(we stopped our experiments at \ni=10). For the two smaller oracles,
the effects of the combinatorial explosion of \splt{P_B}{P_C} become
visible when $\ni=5$ and Hermit turns out to be faster. The growth of response time
is rather slow until $\ni=5$ because the cost caused by the
exponential inflation of $P_B$ is dominated by the cost of oracle
calls, that do not suffer from the combinatorial explosion as
explained in the next paragraph.

\begin{table}[h]
  \caption{Hermit vs IBQ PLR c: average time per subsumption check (ms) }
  \label{tab-comp}
    \small
    \begin{tabular}{lrrrrrrrr}
      \hline
      &    &   \multicolumn{7}{c}{\ni (\# intervals per simple policy)}  
      \\
            &  \hspace{15pt}   & \hspace{15pt} 0 & \hspace{10pt} 1 &  \hspace{10pt}  2&  \hspace{10pt}  3&  \hspace{10pt}  4&  \hspace{10pt} 5&  \hspace{10pt} 6 \\
      \hline
      \hline
        MHA & Hermit & 143.38 & 154.56 & 166.10 & 165.9 & 184.7 & 166.75 & 165.96 
        \\
         & IBQ PLR c  & 21.77 & 25.72 & 30.90 & 36.84 & 68.27 &188.32 & 871.97 
         \\
         &   \%& 15.18 & 16.64 & 18.60 & 22.20 & 36.96 & 112.93 & 525.40 
        \\[2pt]
        \hline
        \hline
         OntolUrgences & Hermit & 150.56 & 162.69 & 168.51 & 176.16 & 167.77 & 179.55 & 171.17 
        \\
         & IBQ PLR c & 33.12 & 36.50 & 41.39 & 52.51 & 74.13 & 184.23 & 1007.84 
         \\
         &  \%& 21.99 & 22.44 & 24.56 & 29.81 & 44.19 & 102.60 & 588.79
        \\[2pt]
        \hline
        \hline
       SNOMED & Hermit & 12806.00 & 14168.52 & 18471.95 & 17730.64 & 18165.25 & 18172.86 & 17588.85 
        \\
         CT & IBQ PLR c & 191.46 & 200.45 & 202.75 & 220.76 & 246.12 & 377.69 & 1057.85
          \\
          &   \%& 1.50 & 1.41 & 1.10 & 1.25 & 1.35 & 2.08 & 6.01 
      \\
      \hline
    \end{tabular}
\end{table}

The effectiveness of the two caches on performance is illustrated in
Table~\ref{tab-opt}. The table illustrates also the explosion caused
by interval splitting as \ni grows. The comparison of the number of
oracle queries issued by the two implementations confirms the
hypothesis that most oracle calls are duplicates caused by interval
splitting (that may create, for each simple policy, multiple versions
that differ only in the intervals). Then the caches keep the number of
oracle queries almost constant.

\begin{table}[h]
  \caption{Effectiveness of optimizations on small/medium and large ontologies}
  \label{tab-opt}
    \small
    \begin{tabular}{lrrr|rr|rrr}
      \hline
                  & \ni & \hspace{5pt} \# disj. & \hspace{10pt} \# disj. &  \hspace{10pt} IBQ PLR & \hspace{10pt} \# oracle  &  \hspace{10pt} IBQ PLR c & \hspace{10pt} \# oracle
      \\
      &   & \hspace{15pt} before & \hspace{10pt} after &  \hspace{10pt} (ms) &  \hspace{10pt} calls &  \hspace{10pt} (ms) & \hspace{10pt} calls &
      \\
            &   & \hspace{15pt} norm & \hspace{10pt} norm &  &   &  &
            
      \\
      \hline
      \hline
        MHA & 0 & 10 & 10 & 130.47 & 1530,23 & 21.77 & 41,64
        \\
         & 1 & 10 & 16,89 & 142.97 & 1756,87 & 25.72 & 41,84
        \\
         & 2 & 10 & 43,03 & 185.76 & 2720,62 & 30.90 & 40,56
        \\
        & 3 & 10 & 153,97 & 340.32 & 6428,17 & 36.84 & 40,98
        \\
         & 4 & 10 & 727,53 & 1043.25 & 23729,95 & 68.27 & 40,67
        \\[2pt]
        \hline
        \hline
        OntolUrgences & 0 & 10 & 10 & 130,80 & 1566,80 & 33,12 & 101,33
        \\
         & 1 & 10 & 16,35 & 145,73 & 1790,14 & 36,50 & 101,34
        \\
         & 2 & 10 & 43,08 & 191,23 & 2783,16 & 41,39 & 101,68
        \\
        & 3 & 10 & 145,19 & 357,16 & 6671,05 & 52,51 & 100,22
        \\
         & 4 & 10 & 653,44 & 1030,05 & 23270,80 & 74,13 & 98,00
        \\[2pt]
        \hline
        \hline
       SNOMED & 0 & 10 & 10 & 359,02& 1597,64 & 191.46 & 104,68
        \\
         CT & 1 & 10 & 16,85 & 395,31& 1845,19 & 200.45 & 106,68
        \\
         & 2 & 10 & 40,32 & 502,34& 2761,74 & 202.75 & 104,64
        \\
        & 3 & 10 & 138,44 & 887,11& 6354,83 & 220.76 & 104,13
        \\
         & 4 & 10 & 677,85 & 2829,61& 25008,38 & 246.12 & 104,54
      \\
      \hline
    \end{tabular}
\end{table}


\section{Related Work}
\label{sec:related}

\PL differs from the tractable profiles of OWL2 and from Horn DLs due
to intervals, that constitute a non-convex domain. \PL without
intervals is a fragment of the tractable DL \emph{Horn-\easycal{SHOIQ}
  with (reuse)-safe} roles
(cf. \cite{DBLP:conf/cade/MartinezFGHH14,DBLP:conf/semweb/CarralFGHH14}).
DLs have been already used as policy languages, e.g.\  \cite{DBLP:conf/policy/UszokBJSHBBJKL03,rei}. These works, however, do not address the encoding of data usage policies, nor the tractability of reasoning.

In
\cite{DBLP:conf/jurix/PalmiraniG18,DBLP:conf/jurix/PalmiraniMRBR18,DBLP:conf/egovis/PalmiraniMRBR18}
the authors propose an ontology, PrOnto, for supporting legal
reasoning and GDPR compliance checking.
Axioms are formulated with nonmonotonic rules interpreted as in corteous logic programming.
The scalability of reasoning is not addressed in these works.
In SPECIAL we favour a DL-based formalization, because it is
particularly well suited to policy comparison (predominant in
SPECIAL's scenarios, and essential for GDPR compliance and data
transfers under sticky policies). In rule-based languages, policy
comparison is generally intractable and even undecidable if rules are
recursive (cf.\ the discussion in \cite{DBLP:conf/datalog/Bonatti10}).
Moreover, SPECIAL is not addressing advanced legal reasoning. For
example, the compliance check w.r.t.\ the GDPR is aimed only at verifying
the policy's internal coherence (e.g.\ does it contain all the
necessary obligations? Is the legal basis appropriate for the data
categories involved?). As a consequence, deontic reasoning and
nonmonotonic reasoning -- frequently adopted in the AI-and-law area --
lie outside the scope of SPECIAL's use cases.


\section{Conclusions}
\label{sec:conclusions}

In summary, IBQ reasoning constitutes an effective approach to
extending \PL reasoning with a wide range of vocabularies, formulated
with more expressive logics. The current experiments are encouraging,
and show that the integration of different reasoners may significantly
increase performance. The benefits of this approach are particularly
visible on very large oracles, such as SNOMED.  This makes the IBQ
approach particularly appealing for reasoning about EHR.

However, the current performance of \PLRO's implementations is not yet
sufficient for SPECIAL's scenarios. To address this issue, we are
going to investigate whether (and to what extent) oracles can be
compiled into \PL knowledge bases, to further speed up reasoning.

In future work we are also going to complete our experimental analysis
by extending the set of test cases, and by evaluating further engines
for querying the oracles, with aprticular attention to the engines
specialized on the Horn fragments of OWL2, such as OWL2-RL and OWL2-QL.


%
%
%
\bibliographystyle{splncs04}
\bibliography{biblio,oracle,bib-prelim,newbiblio}

\begin{thebibliography}{10}
\providecommand{\url}[1]{\texttt{#1}}
\providecommand{\urlprefix}{URL }
\providecommand{\doi}[1]{https://doi.org/#1}

\bibitem{DBLP:conf/dlog/2003handbook}
Baader, F., Calvanese, D., McGuinness, D.L., Nardi, D., Patel-Schneider, P.F.
  (eds.): The Description Logic Handbook: Theory, Implementation, and
  Applications. Cambridge University Press (2003)

\bibitem{DBLP:conf/datalog/Bonatti10}
Bonatti, P.A.: Datalog for security, privacy and trust. In: Datalog Reloaded -
  First International Workshop, Datalog 2010, Oxford, UK, March 16-19, 2010.
  Revised Selected Papers. Lecture Notes in Computer Science, vol.~6702, pp.
  21--36. Springer (2010)

\bibitem{DBLP:conf/ijcai/Bonatti18}
Bonatti, P.A.: Fast compliance checking in an {OWL2} fragment. In: Proceedings
  of the Twenty-Seventh International Joint Conference on Artificial
  Intelligence, {IJCAI} 2018. pp. 1746--1752. ijcai.org (2018)

\bibitem{DBLP:conf/bigdata/BonattiK19}
Bonatti, P.A., Kirrane, S.: Big data and analytics in the age of the {GDPR}.
  In: 2019 {IEEE} International Congress on Big Data, BigData Congress 2019.
  pp. 7--16. {IEEE} (2019)

\bibitem{DBLP:conf/dlog/BonattiPS15}
Bonatti, P.A., Petrova, I.M., Sauro, L.: Optimized construction of secure
  knowledge-base views. In: Proceedings of the 28th International Workshop on
  Description Logics. {CEUR} Workshop Proceedings, vol.~1350. CEUR-WS.org
  (2015), \url{http://ceur-ws.org/Vol-1350/paper-44.pdf}

\bibitem{DBLP:conf/cade/MartinezFGHH14}
Carral, D., Feier, C., Grau, B.C., Hitzler, P., Horrocks, I.: \emph{EL}-ifying
  ontologies. In: Automated Reasoning - 7th International Joint Conference,
  {IJCAR} 2014. Proceedings. pp. 464--479 (2014)

\bibitem{DBLP:conf/semweb/CarralFGHH14}
Carral, D., Feier, C., Grau, B.C., Hitzler, P., Horrocks, I.: Pushing the
  boundaries of tractable ontology reasoning. In: The Semantic Web - {ISWC}
  2014 - 13th International Semantic Web Conference, Proceedings, Part {II}.
  pp. 148--163 (2014)

\bibitem{DBLP:journals/jair/GrauM12}
{Cuenca Grau}, B., Motik, B.: Reasoning over ontologies with hidden content:
  The import-by-query approach. J. Artif. Intell. Res.  \textbf{45},  197--255
  (2012)

\bibitem{DBLP:conf/ijcai/GrauMK09}
{Cuenca Grau}, B., Motik, B., Kazakov, Y.: Import-by-query: Ontology reasoning
  under access limitations. In: {IJCAI} 2009, Proceedings of the 21st
  International Joint Conference on Artificial Intelligence. pp. 727--732
  (2009)

\bibitem{DBLP:conf/kr/HorrocksKS06}
Horrocks, I., Kutz, O., Sattler, U.: The even more irresistible {SROIQ}. In:
  Proceedings, Tenth International Conference on Principles of Knowledge
  Representation and Reasoning. pp. 57--67. {AAAI} Press (2006)

\bibitem{rei}
Kagal, L., Finin, T.W., Joshi, A.: A policy language for a pervasive computing
  environment. In: 4th IEEE International Workshop on Policies for Distributed
  Systems and Networks (POLICY). pp.~63--. IEEE Computer Society (Jun 2003)

\bibitem{DBLP:conf/esws/KirraneFDMPBWDR18}
Kirrane, S., Fern{\'{a}}ndez, J.D., Dullaert, W., Milosevic, U., Polleres, A.,
  Bonatti, P.A., Wenning, R., Drozd, O., Raschke, P.: A scalable consent,
  transparency and compliance architecture. In: The Semantic Web: {ESWC} 2018
  Satellite Events, Revised Selected Papers. Lecture Notes in Computer Science,
  vol. 11155, pp. 131--136. Springer (2018)

\bibitem{DBLP:conf/kr/OrtizRS10}
Ortiz, M., Rudolph, S., Simkus, M.: Worst-case optimal reasoning for the
  {Horn-DL} fragments of {OWL} 1 and 2. In: Principles of Knowledge
  Representation and Reasoning: Proceedings of the Twelfth International
  Conference, {KR} 2010. {AAAI} Press (2010)

\bibitem{DBLP:conf/ijcai/OrtizRS11}
Ortiz, M., Rudolph, S., Simkus, M.: Query answering in the horn fragments of
  the description logics {SHOIQ} and {SROIQ}. In: {IJCAI} 2011, Proceedings of
  the 22nd International Joint Conference on Artificial Intelligence. pp.
  1039--1044. {IJCAI/AAAI} (2011)

\bibitem{DBLP:conf/jurix/PalmiraniG18}
Palmirani, M., Governatori, G.: Modelling legal knowledge for {GDPR} compliance
  checking. In: Legal Knowledge and Information Systems - {JURIX} 2018: The
  Thirty-first Annual Conference. pp. 101--110 (2018)

\bibitem{DBLP:conf/jurix/PalmiraniMRBR18}
Palmirani, M., Martoni, M., Rossi, A., Bartolini, C., Robaldo, L.: Legal
  ontology for modelling {GDPR} concepts and norms. In: Legal Knowledge and
  Information Systems - {JURIX} 2018: The Thirty-first Annual Conference. pp.
  91--100 (2018)

\bibitem{DBLP:conf/egovis/PalmiraniMRBR18}
Palmirani, M., Martoni, M., Rossi, A., Bartolini, C., Robaldo, L.: Pronto:
  Privacy ontology for legal reasoning. In: Electronic Government and the
  Information Systems Perspective - 7th International Conference, {EGOVIS}
  2018, Proceedings. pp. 139--152 (2018)

\bibitem{DBLP:books/daglib/0018514}
Papadimitriou, C.H.: Computational complexity. Academic Internet Publ. (2007)

\bibitem{DBLP:conf/policy/UszokBJSHBBJKL03}
{Uszok, A.\ et al.}: {KAoS} policy and domain services: Towards a
  description-logic approach to policy representation, deconfliction, and
  enforcement. In: 4th IEEE Int.\ Work.\ on Policies for Distributed Systems
  and Networks (POLICY). pp. 93--96. IEEE Computer Society (2003)

\end{thebibliography}
\end{document}